\def\year{2018}\relax
\documentclass[letterpaper]{article} 
\usepackage{aaai18}  
\usepackage{times}  
\usepackage{helvet}  
\usepackage{courier}  
\usepackage{url}  
\usepackage{graphicx}  
\usepackage{times}
\usepackage{helvet}
\usepackage{courier}

\usepackage{enumerate}
\usepackage{subcaption}
\usepackage{amssymb}
\usepackage{amsmath}
\usepackage{algorithm}
\usepackage{algorithmic}
\usepackage{multirow}
\usepackage{bm}
\usepackage{amsthm}
\usepackage{graphicx,mathtools}

\usepackage{url}
\usepackage{caption,color}
\usepackage{pifont}

\newtheorem{theorem}{Theorem}
\newtheorem{lem}{Lemma}
\newtheorem{definition}{Definition}

\newcommand{\hatU}{\hat{U}}   
\newcommand{\UU}{UU^\top}   
\newcommand{\UUk}{U_kU_k^\top}   
\newcommand{\UUkk}{U_{k+1}U_{k+1}^\top}   
\newcommand{\Uk}{U_k}   
\newcommand{\Pk}{P_k}   
\newcommand{\Yk}{Y_k}   
\newcommand{\Ukk}{U_{k+1}}   
\newcommand{\Pkk}{P_{k+1}}   
\newcommand{\Ykk}{Y_{k+1}}   
\newcommand{\mukk}{\mu_{k+1}}   
\newcommand{\muk}{\mu_k}   
\newcommand{\gammak}{\gamma_k}   
\newcommand{\0}{\ensuremath{\mathbf{0}}}
\newcommand{\1}{\ensuremath{\mathbf{1}}}

\newcommand{\norm}[1]{\left\|#1\right\|}
\newcommand{\inproduct}[2]{\left\langle#1,#2\right\rangle}

\DeclareMathOperator*{\argmin}{argmin}

\newcommand{\Tr}{\text{Tr}}

\newcommand{\st}{\text{s.t.}}

\newcommand{\Lag}{\mathcal{L}}

\begin{document}
%
\title{Nonconvex Sparse Spectral Clustering by Alternating Direction Method of Multipliers and Its Convergence Analysis}

\author{Canyi Lu$^1$, Jiashi Feng$^1$, Zhouchen Lin$^{2,3}$\thanks{Corresponding author.}, Shuicheng Yan$^{4,1}$\\
	$^1$ Department of Electrical and Computer Engineering, National University of Singapore\\
	$^{2}$ Key Laboratory of Machine Perception (MOE), School of EECS, Peking University\\
	$^3$ Cooperative Medianet Innovation Center, Shanghai Jiao Tong University \\
	$^4$ 360 AI Institute \\
	{\tt\small canyilu@gmail.com, elefjia@nus.edu.sg, zlin@pku.edu.cn, eleyans@nus.edu.sg}
}

\maketitle
\begin{abstract}
	Spectral Clustering (SC) is a widely used  data clustering method which first learns a low-dimensional embedding $U$ of data by computing the eigenvectors of the normalized Laplacian matrix, and then performs k-means on $U^\top$ to get the final clustering result. The Sparse Spectral Clustering (SSC) method extends SC with a sparse regularization on $\UU$ by using the block diagonal structure prior of $\UU$ in the ideal case. However, encouraging  $\UU$ to be sparse leads to a heavily  nonconvex problem which is challenging to solve and the work \cite{lu2016convex}  proposes a convex relaxation in the pursuit of this aim indirectly. However, the convex relaxation generally leads to a loose approximation and the quality of the solution is not clear.  This work instead considers to solve the nonconvex formulation of SSC  which directly encourages  $\UU$ to be sparse. We propose an efficient Alternating Direction Method of Multipliers (ADMM) to solve the nonconvex SSC and provide the convergence guarantee. In particular, we prove that the sequences generated by ADMM always exist a limit point and any limit point is a stationary point. Our analysis does not impose any assumptions on the iterates and thus is practical. Our proposed ADMM for nonconvex problems allows the stepsize to be increasing but upper bounded, and this makes it very efficient in practice. Experimental analysis on several real data sets verifies the effectiveness of our method.
\end{abstract}

\section{Introduction}

Data clustering is one of the most fundamental topics in unsupervised learning and has been widely applied in computer vision, data mining and many others. Clustering aims to divide the unlabeled data
set into groups which consist of similar data points. 
Many clustering methods have been proposed up to now,
e.g. k-means, spectral clustering \cite{ng2002spectral,shi2000normalized} and nonnegative
matrix factorization \cite{lee2001algorithms}. The Spectral Clustering (SC) is  one of the most widely used  methods and it has a lot of applications in computer vision and signal analysis, e.g., image segmentation \cite{shi2000normalized}, motion segmentation \cite{lu2012robust}, and co-clustering problems of words and documents \cite{dhillon2001co}.  

Assume that we are given   $n$ data points $X=[x_1,\cdots,x_n]=[X_1,\cdots,X_k]\in\mathbb{R}^{d\times n}$, where $X_j\in\mathbb{R}^{d\times n_j}$ denotes the $j$-th group  with $n_j$ points, $\sum_{j=1}^kn_j=n$ and $k$ is the number of clusters. SC \cite{ng2002spectral} partitions these $n$ points into $k$ clusters by the following procedures: First, 
compute an affinity matrix $W\in\mathbb{R}^{n\times n}$ with its element $w_{ij}$ measuring the similarity between $x_i$ and $x_j$. Second, construct the normalized Laplacian matrix $L=I-D^{-\frac{1}{2}}W D^{-\frac{1}{2}}$, where $D$ is a diagonal matrix with each diagonal element $d_{ii}=\sum_{j=1}^nw_{ij}$ and $I$ is the identity matrix. Third, compute $U\in\mathbb{R}^{n\times k}$ by solving 
\begin{equation}\label{sc}
\min_{U\in\mathbb{R}^{n\times k}}  \  \langle UU^\top,L\rangle, \  \st \    U^\top U=I.
\end{equation}
Finally, compute $\hat{ U}\in\mathbb{R}^{n\times k}$ by normalizing each row of $ U$ to have unit Euclidean length,  treat the rows of $\hat{ U}$ as data points in $\mathbb{R}^k$, and cluster them into $k$ groups by k-means.
Due to the significance of SC, many variants of SC have been proposed based on different ways of affinity matrix $W$ construction and  different normalizations of the Laplacian matrix $L$  \cite{shi2000normalized,von2007tutorial}. 

A recent work \cite{lu2016convex} proposes the Sparse Spectral Clustering (SSC) method which  computes the low-dimensional embedding $U$ in a different way:
\begin{align}\label{ssc0}
\min_{U\in\mathbb{R}^{n\times k}} \  \inproduct{L}{\UU} + \beta \norm{\UU}_0, \ \st \ U^\top U=I,
\end{align}
where $\beta>0$ and  $\norm{\cdot}_0$ is the $\ell_0$-norm  which encourages  $\UU$ to be sparse. The motivation for such a sparse regularizer is that $\UU$ is  block diagonal (thus sparse) when $W$ is block diagonal in the ideal case. 
Consider the ideal case that the affinity matrix $W$ is block diagonal, i.e., $w_{ij}=0$ if $x_i$ and $x_j$ are from different clusters. Let  $C\in\mathbb{R}^{n\times k}$ denotes the indicator matrix whose row entries indicate to which group the points belong. That is, if $x_i$ belongs to the group $l$, $c_{il}=1$ and $c_{ij}=0$ for all $j\neq l$.  Then, for  any orthogonal matrix $R\in\mathbb{R}^{k\times k}$,  we have $\hat{U}=C R$.
In this case,  $ \hatU\hatU^\top$ is block diagonal, i.e.,
\begin{equation*}
\hatU\hatU^\top=CC^\top=\begin{bmatrix}
\1_{n_1}\1_{n_1}^\top & \0 & \cdots  & \0 \\
\0 & \1_{n_2}\1_{n_2}^\top & \cdots & \0 \\
\vdots & \vdots & \ddots & \vdots \\
\0 & \0 & \cdots & \1_{n_k}\1_{n_k}^\top
\end{bmatrix},
\end{equation*}
where $\1_m$ denotes the all one vector of length $m$ and $\0$ is all one vector/matrix of proper size. 
Hence,     $\hatU\hatU^\top$ implies  the true membership of the data clusters and it is naturally {sparse}. Note that $\hat{U}$ is obtained by normalizing each row of $U$ and thus $\UU$ is also sparse. However, such a block diagonal or sparse property may not appear in real applications since the affinity matrix $W$ is usually not block diagonal. This motivates the sparse regularization on $\UU$ and thus leads to the SSC model in (\ref{ssc0}). However, the key challenge is that problem (\ref{ssc0}) is nonconvex and difficult to solve. The work \cite{lu2016convex} proposes a convex relaxation formulation as follows
\begin{equation}\label{cssc}
\min_{ P\in\mathbb{R}^{n\times n}}        \langle P,L\rangle+\beta\|P\|_1, \ \ \text{s.t.} \  \   \0\preceq P\preceq I,\ \Tr(P)=k,
\end{equation}
where the $\ell_1$-norm $\norm{\cdot}_1$ is used as a surrogate of $\ell_0$-norm while the nonconvex constraint consisting of all  the fixed rank projection matrices, i.e., $\{\UU|U^\top U=I\}$, is replaced as its convex hull $\{P\in\mathbb{S}^{n\times n}| 0\preceq P\preceq I, \Tr(P)=k\}$ \cite{fillmore1971some}.  Here, $\mathbb{S}^n$ denotes the set of symmetric matrices. For $A,B\in\mathbb{S}^n$, $A\preceq B$  means that $B-A$ is positive semi-definite. Problem (\ref{cssc}) is convex and the optimal solution can be computed by Alternating Direction Method of Multipliers (ADMM) \cite{gabay1976dual,lu2017unified}, which is efficient in practice. 
After solving (\ref{cssc}) with the solution $P^*$, the low-dimensional embedding $U$ of data $X$ can be approximated by using the first $k$ eigenvectors corresponding to the largest $k$ eigenvalues of $P$. This is equivalent to computing $U$ by solving
\begin{equation}\label{computUssc}
\min_{U\in\mathbb{R}^{n\times k}} \ \norm{P^* - \UU}, \ \st \ U^\top U = I,
\end{equation}
where $\norm{\cdot}$ denotes the Frobenius norm of a matrix.
After obtaining $U$, one is able to cluster the data points into $k$ groups as that in SC.

From the above discussions, it can be seen that a main limitation of the convex SSC relaxation  (\ref{cssc}) is that the obtained solution may be far from optimal to  (\ref{ssc0}). The reason is that the difference $\norm{P^* - \UU}$ in (\ref{computUssc}) is not guaranteed to be 0 or sufficiently small. Thus, enforcing $P$ to be sparse in (\ref{cssc}) does not guarantee a sparse $\UU$ in (\ref{computUssc}). It is obvious that such an issue is mainly caused by the usage of the relaxation of  the convex hull $\{P\in\mathbb{S}^{n\times n}| 0\preceq P\preceq I, \Tr(P)=k\}$ instead of $\{\UU|U^\top U=I\}$. 

In this work, we aim to address the above possibly loose relaxation issue of the convex SSC model by   directly solving the following nonconvex problem
\begin{align}\label{ncvxssc}
\min_{U\in\mathbb{R}^{n\times k}} \  \inproduct{L}{\UU} + g(\UU), \ \st \ U^\top U=I,
\end{align}
where $g:\mathbb{R}^{n\times n}\rightarrow \mathbb{R}$ is a sparse regularizer. The choice of the spare regularizer is not very important in SSC.  We allow  $g$ to be nonconvex (see more detailed assumption in the next section). Problem (\ref{ncvxssc}) is  nonconvex and challenging to solve due to the orthogonal constraint and the complex term $g(\UU)$. This work  proposes to solve it by  Alternating Direction Method of Multipliers (ADMM)  and provide the convergence guarantee. Our proposed ADMM is flexible as we allow the stepsize in ADMM to be monotonically increasing (but upper bounded). Such a choice of stepsize is widely used and has been verified to be effective in improving the convergence of ADMM for convex optimization. Note that we are the first one to use such a choice of stepsize in ADMM for nonconvex problems and provide its support in theory. When characterizing the convergence of ADMM, we show that the augmented Lagrangian function value is monotonically decreasing. This further guarantees that  the sequences generated by the proposed ADMM are bounded and there always exists a limit point and any limit point is a stationary point.  


\begin{algorithm}[t]
	\caption{Solve (\ref{pro}) by ADMM}
	\textbf{Initialize:} $\rho>1$, $\mu_{\max}$, $k=0$, $\Pk$, $\Uk$, $\Yk$, $\muk$. \\
	\textbf{while} not converged \textbf{do}
	\begin{enumerate}
		\item Compute $\Ukk$ by solving (\ref{updateU});
		\item Compute $\Pkk$ by solving (\ref{updateP});
		\item Compute $\Ykk$ by (\ref{updateY});
		\item Compute $\mukk$ by (\ref{updatemu});
		\item $k=k+1$.
	\end{enumerate}
	\textbf{end while}\label{admm}
\end{algorithm}

\section{The Proposed ADMM Algorithm}
In this section,  we present the ADMM algorithm for solving the nonconvex problem (\ref{ncvxssc}). We first reformulate it as the following equivalent problem
\begin{equation}\label{pro}
\begin{split}
&\min_{P\in\mathbb{R}^{n\times n},U\in\mathbb{R}^{n\times k}} \  \inproduct{L}{\UU} + g(P), \\
& \st \ P = \UU, \ U^\top U=I. 
\end{split}
\end{equation}
The standard  augmented Lagrangian function is
\begin{align*}
&\Lag(P,U,Y_1,Y_2,\mu) =   \inproduct{L}{\UU} + g(P) + \inproduct{Y_1}{P-\UU}\notag \\
&+\inproduct{Y_2}{\UU-I} + \frac{\mu}{2} \norm{P-\UU}^2 + \frac{\mu}{2} \norm{U^\top U-I}^2, 
\end{align*}
where $Y_1$ and $Y_2$ are the dual variables and $\mu>0$. However, it is difficult to update $U$ by minimizing the above augmented Lagrangian function when fixing other variables. To update $U$ efficiently, we instead use the following
\emph{partial} augmented Lagrangian function  
\begin{align}
\Lag(P,U,Y,\mu) = &  \inproduct{L}{\UU} + g(P) + \inproduct{Y}{P-\UU}\notag \\
& + \frac{\mu}{2} \norm{P-\UU}^2.
\end{align}
Then we can solve problem (\ref{pro}) by Alternating Direction Method of Multipliers by the following rules.

\noindent 1. Fix $P=\Pk$ and update $U$ by
\begin{align}
\Ukk = & \argmin_{U\in\mathbb{R}^{n\times k}} \ \Lag(\Pk,U,\Yk,\muk), \ \st \ U^\top U=I. \notag \\
= & \argmin_{U} \ \norm{\UU - \Pk +(L-\Yk)/\muk  }^2, \label{updateU}\\
& \st \ U^\top U=I.   \notag
\end{align}
2. Fix $U=\Ukk$ and update $P$ by
\begin{align}
\Pkk =& \argmin_{P}  \ \Lag(P,\Ukk,\Yk,\muk) \notag \\
= & \argmin_{P} \ g(P) + \frac{\muk}{2} \norm{P-\UUkk+\Yk/\muk}^2. \label{updateP}
\end{align}
3. Update the dual variable by
\begin{equation}\label{updateY}
\Ykk = \Yk + \muk (\Pkk - \UUkk).
\end{equation}
4. Update the stepsize $\mu$ by
\begin{equation}\label{updatemu}
\mukk = \min(\mu_{\max},\rho\muk), \ \rho > 1. 
\end{equation}
The whole procedure of ADMM for (\ref{pro}) is given in Algorithm \ref{admm}. It can be seen that the $U$-subproblem (\ref{updateU}) has a closed form solution. The $P$-subproblem (\ref{updateP}) requires computing the proximal mapping of $g$. It   usually has a closed form solution when $g$ is simple.

We would like to emphasize that, for nonconvex optimization, our ADMM  allows the stepsize $\muk$ to be increasing (but upper bounded), while previous nonconvex ADMM methods have to fix it. Such a choice of stepsize has been verified to be effective in improving the convergence of ADMM for convex optimization  in practice and the convergence guarantee is also known \cite{lu2017unified,lin2010augmented}. To the best of our knowledge, this is the first work which uses varying stepsize in ADMM for nonconvex problems and the convergence analysis for supporting this is also different from the convex case.



\section{Main Result: The Convergence Analysis}
The most important contribution is the convergence analysis of the proposed  ADMM in Algorithm \ref{admm} for nonconvex problems which is generally challenging. This section gives the details of the convergence results. 
We first introduce the subgradient of any function   \cite{rockafellar2009variational}, which will be used later.

\begin{definition}
	Let $S\subseteq \mathbb{R}^m$ and ${x}_0\in S$. A vector $v$ is normal to $S$ at $x_0$ in the regular sense, denoted as $v\in \hat{N}_S(x_0)$, if
	\begin{equation*}
	\inproduct{v}{x-x_0} \leq o(\norm{x-x_0}), \ x\in S,
	\end{equation*}
	where $o(\norm{y})$ is defined by $\lim\limits_{\norm{y}\rightarrow0} \frac{o(\norm{y})}{\norm{y}}=0$. A vector is normal to $S$ at $x_0$ in the general sense, denoted as $v\in N_S(x_0)$, if there exist sequences $\{x^k\}\subset S$, $\{v^k\}$ such that $x^k\rightarrow x_0$ and $v^k\rightarrow v$ with $v^k\in \hat{N}_S(x^k)$. The cone $N_S(x_0)$ is called the normal cone to $S$ at $x_0$.
\end{definition}

\begin{definition}
	Consider a lower semi-continuous function $h: \mathbb{R}^m\rightarrow (-\infty,+\infty]$ and a point $x_0$ with $h(x_0)$ finite. For a vector $v\in\mathbb{R}^m$, one says that
	\begin{enumerate}[(a)]
		\item $v$ is a regular subgradient of $h$ at $x_0$, denoted as $v\in\hat{\partial} h(x_0)$, if
		\begin{equation*}
		h(x) \geq h(x_0) + \inproduct{v}{x-x_0} + o(\norm{x-x_0});
		\end{equation*}
		\item $v$ is a (general) subgradient of $h$ at $x_0$, denoted as $v\in \partial h(x_0)$, if there exist sequences $\{x^k\}$, $\{v^k\}$ such that $x^k\rightarrow x_0$, $h(x_k)\rightarrow h(x_0)$ and $v^k\in \hat{\partial} h(x^k)$ with $v^k\rightarrow v$.
	\end{enumerate}
\end{definition}

Let $S$ be a closed non-empty subset of $\mathbb{R}^m$ and its  indicator function be
\begin{equation*}
\iota_S(x)  = \begin{cases}
0, \qquad \text{if } x \in S,\\
+\infty,\quad \text{otherwise.}
\end{cases}
\end{equation*}
Then its subgradient is
$\partial \iota_S(x_0) = N_S(x_0), x_0\in S.$
In this work, we denote $\mathcal{O} = \{ U\in\mathbb{R}^{n\times k} | U^\top U = I\}$ and the indicator function  as   $\iota_{\mathcal{O}}(U)$.

To guarantee the convergence, we need some assumptions for problem (\ref{pro}) as follows: 
\begin{itemize} 
	\item[\textbf{ A1.}] $L$ is positive semi-definite.
	\item[\textbf{A2.}]$g: \mathbb{R}^{n\times n}\rightarrow R$ is lower bounded, differential and $\nabla g$ is Lipschitz continuous, i.e., there exists $ l >0$ such that
	\begin{equation*}
	\norm{\nabla g(X) - \nabla g(Y)} \leq l \norm{X-Y}, \ \forall X,Y\in\mathbb{R}^{n\times n}.
	\end{equation*}
	\item[\textbf{A3.}] The stepsize $\muk$ is chosen large enough such that 
	\begin{enumerate}[(1)]
		\item The $P$-subproblem (\ref{updateP}) is strongly convex with modulus $\gammak$.
		\item $\muk\gammak > l^2(\rho+1)$ and $\muk\geq l$.
	\end{enumerate}
\end{itemize}

We have some remarks regarding   the above assumptions. First,  A1 holds for the SSC model since $L$ is the normalized Laplacian matrix; Second, $g$ can be nonconvex. In SSC, one may use convex or nonconvex sparse regularizer $g$. But $g$ should be Lipschitz differential which can be achieved by using the smoothing technique \cite{nesterov2005smooth} if necessary (see an example in the Experiment section); Third, the  $P$-subproblem (\ref{updateP}) is eventually strongly convex when $\muk$ is large enough.

At the following, we will prove several properties of Algorithm \ref{admm} and give the convergence results.

\begin{lem}\label{lema1}
	Under assumptions A1-A3, all the subproblems in Algorithm \ref{admm} are well defined.
\end{lem}
\begin{proof}
	The $P$-subproblem (\ref{updateP}) is well defined since $g$ is lower bounded under assumption A2. Also, it is obvious that the $U$-subproblem (\ref{updateU}) is well defined.
\end{proof}

\begin{lem}\label{lem2}
	Under assumptions A1-A3, we have
	\begin{equation}\label{boundYdiff}
	\norm{\Yk-\Ykk}^2 \leq l^2\norm{\Pk-\Pkk}^2.
	\end{equation}
\end{lem}
\begin{proof}
	From the $P$-subproblem  (\ref{updateP}), we have the following optimality condition
	\begin{equation}\label{optcondPP}
	\nabla g(\Pkk) + \Yk + \muk(\Pkk-\UUkk) = 0.
	\end{equation}
	By using $\Ykk = \Yk + \muk (\Pkk - \UUkk)$ in (\ref{updateY}), we have 
	\begin{equation}\label{optcondP}
	\nabla g(\Pkk) =-\Ykk.
	\end{equation}
	Then we have
	\begin{align*}
	\norm{\Ykk-\Yk} = \norm{\nabla g(\Pkk) - \nabla g(\Pk)} \leq l \norm{\Pkk-\Pk},
	\end{align*}
	where the last inequality uses assumption A2.
	The proof is completed.
\end{proof}


\begin{lem}\label{lem3}
	Under assumptions A1-A3, the sequences $\{\Pk,\Uk,\Yk\}$ generated by Algorithm \ref{admm} satisfy
	\begin{itemize}
		\item[(a)] $\Lag(\Pk,\Uk,\Yk,\muk)$ is monotonically decreasing, i.e.,
		\begin{align}
		&\Lag(\Pkk,\Ukk,\Ykk,\mukk) - \Lag(\Pk,\Uk,\Yk,\muk)   \notag \\ 
		\leq &  - \left( \frac{\gammak}{2} - \frac{l^2(\rho+1)}{2\muk}  \right) \norm{\Pkk-\Pk}^2. \label{profdecresL}
		\end{align}
		\item[(b)] $\lim\limits_{k\rightarrow +\infty}\Lag(\Pk,\Uk,\Yk,\muk)=\Lag^*$ for some constant $\Lag^*$. 
		\item[(c)] When $k\rightarrow+\infty$, $\Pkk-\Pk\rightarrow0$,  $\Ykk-\Yk\rightarrow0$ and $\Pk-\UUk\rightarrow0$. 
		\item[(d)] The sequences $\{\Pk\}$, $\{\Uk \}$ and $\{\Yk \}$ are bounded.
		\item[(e)] There exists $G=[G_P \  G_U \ G_Y]$, where 
		\begin{align*}
		G_P =& \partial_P \Lag(\Pkk,\Ukk,\Ykk,\muk),\\
		G_U \in& \partial_U \Lag(\Pkk,\Ukk,\Ykk,\muk)+\partial_U \iota_{\mathcal{O}}(\Ukk),\\
		G_Y = &\partial_Y \Lag(\Pkk,\Ukk,\Ykk,\muk),
		\end{align*}
		
		such that
		\begin{align}
		\norm{G}^2 \leq & (8d+1+\frac{1}{\mu_0^2})\norm{\Yk-\Ykk}^2 \notag\\
		&+ 8d\mu_{\max}^2\norm{\Pk-\Pkk}^2.\label{eqboundsubgr}
		\end{align}
	\end{itemize}
\end{lem}
The proof of Lemma \ref{lem3}  can be found  in the supplementary material. The property (\ref{boundYdiff}) is important for proving (\ref{profdecresL}), which guarantees that $\Lag(\Pk,\Uk,\Yk,\muk)$ is monotonically decreasing due to the choice of $\muk$ in assumption A3. This combined with the lower bounded property guarantees that   $\Lag(\Pk,\Uk,\Yk,\muk)$ converges to some $\Lag^*>0$. For convex problems, there have several different quantities to characterize the convergence of ADMM, see  \cite{lu2017unified,he20121,liu2013linearized}. However, they are not applicable to ADMM for nonconvex problems. Here, the convergence is characterizing based on a different way by using the decreasing sequence   $\Lag(\Pk,\Uk,\Yk,\muk)$. Note that the decreasing property of  $\Lag(\Pk,\Uk,\Yk,\muk)$ does not necessary hold for ADMM for convex optimization. This difference implies that the nonconvex problems which can be solved by ADMM  are relatively limited and the convergence guarantee of ADMM for nonconvex problems is much more challenging. Based on (\ref{profdecresL}), many other properties are proved. For example, Lemma \ref{lem3} (c) gives some necessary results when the algorithm converges and (d) and (e) are important for proving the convergence to stationary point shown  below.

\begin{theorem}\label{thm1} Assume that the assumptions A1-A3 are satisfied.
	Let $(P^*,U^*,Y^*)$ denotes any limit point of the sequence $\{\Pk,\Uk,\Yk\}$ generated by Algorithm \ref{admm}. Then the limit point is a stationary point of problem (\ref{pro}), i.e.,
	\begin{align}
	& 0\in \partial_U\Lag(P^*,U^*,Y^*,\mu^*)  +\partial_U \iota_{\mathcal{O}}(U^*),\label{eqnthm1stapoint1}\\
	& 0  = \partial_P\Lag(P^*,U^*,Y^*,\mu^*),\label{eqnthm1stapoint2}\\
	& 0  = \partial_Y\Lag(P^*,U^*,Y^*,\mu^*) = P^* - U^*{U^*}^\top.\label{eqnthm1stapoint3}
	\end{align}
\end{theorem}
\begin{proof} From the boundedness of $\{\Pk,\Uk,\Yk,\muk\}$ in Lemma \ref{lem3}, there exists a convergent subsequence and a  limit point, denoted by  $(P_{k_j}, U_{k_j},Y_{k_j},\mu_{k_j}) \rightarrow (P^*,U^*,Y^*,\mu^*)$ as $j\rightarrow +\infty$.  Then, by using $\Pkk-\Pk\rightarrow0$,  $\Ykk-\Yk\rightarrow0$ and (\ref{eqboundsubgr}) in Lemma \ref{lem3}, for $k\geq1$, there exists $G_{k}\in\partial \Lag(\Pk,\Uk,\Yk,\mu_{k-1})$ such that $\norm{G_k}\rightarrow 0$. In particular, $\norm{G_{k_j}}\rightarrow 0$ as $j\rightarrow+\infty$. By the definition of general subgradient, we have $0\in \partial\Lag(P^*,U^*,Y^*,\mu^*)$. This implies that (\ref{eqnthm1stapoint1})-(\ref{eqnthm1stapoint3}) hold. Thus, any limit point is a stationary point.
\end{proof}

\begin{theorem}\label{thm2}
	For every $K\geq 1$, the sequences $\{\Pk,\Uk,\Yk\}$ generated by Algorithm \ref{admm} satisfies
	\begin{align*}
	\min_{0\leq k\leq K} \norm{\Pkk-\Pk}^2 &\leq \frac{\Lag(P_0,U_0,Y_0,\mu_0) - \Lag^*}{(K+1)c_K},\\
	\min_{0\leq k\leq K} \norm{\Ykk-\Yk}^2 &\leq \frac{l^2(\Lag(P_0,U_0,Y_0,\mu_0) - \Lag^*)}{(K+1)c_K}, \\
	\min_{0\leq k\leq K} \norm{\Pkk-\UUkk}^2 &\leq \frac{l^2(\Lag(P_0,U_0,Y_0,\mu_0) - \Lag^*)}{(K+1)c_K\mu_0^2},
	\end{align*}
	where $c_K = \min_{0\leq k\leq K} \left( \frac{\gammak}{2} - \frac{l^2(\rho+1)}{2\muk} \right)>0$.
\end{theorem}
\begin{proof}
	From  (\ref{profdecresL}) and the  definition of $c_K>0$, we have
	\begin{align*}
	& c_K \norm{\Pkk-\Pk}^2 \\
	\leq  &\Lag(\Pk,\Uk,\Yk,\muk) -  \Lag(\Pkk,\Ukk,\Ykk,\mukk).   
	\end{align*}
	Summing the above equality over $k=0,\cdots,K$, we obtain
	\begin{align*}
	&\sum_{k=0}^{K} c_K \norm{\Pkk-\Pk}^2 \\
	\leq &  \Lag(P_0,U_0,Y_0,\mu_0) -  \Lag(\Pkk,\Ukk,\Ykk,\mukk)  \\
	\leq & \Lag(P_0,U_0,Y_0,\mu_0) -  \Lag^*.
	\end{align*}
	Thus, we have
	\begin{align*}
	&\min_{0\leq k\leq K} \norm{\Pkk-\Pk}^2 \\
	\leq & \frac{1}{K+1} \sum_{k=0}^{K}  \norm{\Pkk-\Pk}^2 
	\leq  \frac{\Lag(P_0,U_0,Y_0,\mu_0) - \Lag^*}{(K+1)c_K}.
	\end{align*}
	The proof is completed by further using (\ref{boundYdiff}) and (\ref{updateY}).
\end{proof}
Theorem \ref{thm2} gives the $O(1/K)$ convergence rate of our proposed ADMM based on the smallest difference of iterates of $\Pk$, $\Yk$ and the residual.  
To the best of our knowledge, this is the first convergence rate of ADMM  for nonconvex problems. In theory, such a result is relatively weaker than the convex case since the used measure is $\min_{0\leq k\leq K} \norm{\Pkk-\Pk}^2$ but not $\norm{P_{K+1}-P_K}^2$. But in practice, we observe that the sequence $\norm{P_{k+1}-P_k}^2$ seems to decrease (see Figure \ref{fig_convergence} (b) in the Experiment section), though this is currently not clear in theory. This observation in practice implies that the above convergence rate makes sense.


It is worth mentioning that  the convergence guarantee of ADMM for convex problems has been well established \cite{lu2017unified,he20121,liu2013linearized}. However, for nonconvex cases, the convergence analysis of ADMM for different nonconvex problems is   quite different. There are some recent works \cite{hong2016convergence,wang2015global} which apply ADMM to solve nonconvex problems and provide some analysis. However, these works are not able to solve our problem (\ref{pro}) since the constraints in their considered problems should be relatively simple while our problem has a special nonconvex constraint $P=\UU$. The work \cite{hong2016convergence} is not applicable to our problem since it requires all the subproblems to be strongly convex while our $U$-subproblem (\ref{updateU}) is nonconvex. When considering to apply the method in \cite{wang2015global} to solve (\ref{ncvxssc}), one has to exactly solve the problem of the following type
\begin{equation*}
\min_{P} \ g(PP^\top) + \frac{1}{2}\norm{P-B}^2,
\end{equation*}
in each iteration.  This is generally very chellenging even when $g$ is convex. Also, we would like to emphasize that our ADMM allows the stepsize $\muk$ to be increasing (but upper bounded), while previous nonconvex ADMM algorithms simply fix it. Though our considered problem is specific, the analysis for the varying stepsize $\muk$ is applicable to other nonconvex problems in \cite{hong2016convergence}. In practice, the convergence speed of ADMM is sensitive to the choice of $\mu$, but it is generally difficult to find a proper constant stepsize  for fast convergence.  Our choice of $\muk$ has been shown to be effective in improving the convergence speed and widely used in convex optimization \cite{lu2017unified,lin2010augmented}. In practice, we find  that such a technique is also very useful for fast implementation  of ADMM  for nonconvex problems. We are also the first one to give the convergence rate (in the sense of Theorem \ref{thm2}) of ADMM  for nonconvex problems.

\section{Experiments}

In this section, we conduct some experiments to analyze the convergence of the proposed ADMM for nonconvex SSC and show its effectiveness for data clustering. We consider to solve the following nonconvex SSC model
\begin{equation}\label{prosmooth}
\begin{split}
&\min_{P\in\mathbb{R}^{n\times n},U\in\mathbb{R}^{n\times k}} \  \inproduct{L}{\UU} + g_\sigma(P), \\
& \st \ P = \UU, \ U^\top U=I,
\end{split}
\end{equation}
where $g_\sigma$ is the smoothed $\ell_1$-norm $\beta \norm{P}_1$ with a smoothness parameter $\sigma>0$ defined as follows 
\begin{equation}
g_\sigma(P) = \max_Z \inproduct{P}{Z} - \frac{\sigma}{2}\norm{Z}^2, \ \st \ \norm{Z}_\infty \leq \beta,
\end{equation}
where $\norm{Z}_\infty=\max_{ij}|z_{ij}|$. According to Theorem 1 in \cite{nesterov2005smooth}, the gradient of $  g_\sigma(P)$ is given by $\nabla g_\sigma(P) = \min\{\beta,\max\{ P/\sigma,-\beta\} \}$ and is Lipschitz continuous with Lipschitz constant $l=1/\sigma$. Note that $g_\sigma$ is convex. So we set 
$\mu_0=1.01(l\sqrt{\rho+1})$, which guarantees the assumption A3 holds. In Algorithm \ref{admm}, we   set $\rho{=1.05}$, $\mu_{\max}=1e10$, and $U_0$ is initialized as the $k$ eigenvectors associated to the $k$ smallest eigenvalues of $L$, where $k$ is the number of the clusters and $L$ is the normalized Laplacian matrix constructed based on the given affinity matrix $W$. Then we set $P_0=U_0U_0^\top$ and $Y_0=0$. We use the following stopping criteria for Algorithm \ref{admm}
\begin{equation}\label{stopc}
\max\{\norm{\Pkk-\Pk}_\infty,\norm{\Pkk-\UUkk}_\infty\} \leq  10^{-6},
\end{equation}
which is implied by our convergence analysis.
For all the experiments, we use $\sigma = 0.01$ (in practice, we find that the clustering performance is not sensitive when $\sigma\leq 0.01$).

We conduct two experiments based on two different ways of affinity matrix construction. The first experiment considers the subspace clustering problem in \cite{elhamifar2013sparse}. A \textit{sparse} affinity matrix $W$ is computed by using the sparse subspace clustering method ($\ell_1$-graph) \cite{elhamifar2013sparse}, and then it is used as the input for SC, convex SSC \cite{lu2016convex}  and our nonconvex SSC model solved by ADMM. The second experiment instead uses the Gaussian kernel to construct the affinity matrix which is generally \textit{dense}. We will show the effectiveness of nonconvex SSC in both settings.

\begin{table}[]
	\small
	\centering
	\caption{  Clustering errors ($\%$) on the Extended Yale B database based on the \textit{sparse} affinity matrix $W$ constructed by the $\ell_1$-graph.} 
		\vspace{-0.2em}
	\label{tab_resyaleb}	
	\begin{tabular}{c|c|c|c|c}
		\hline
		$\#$ of            & \multirow{2}{*}{SC}       & \multirow{2}{*}{SSC}  & SSC- & SSC-  \\
		subjects &     &      & PG & ADMM                       \\ \hline	
		2  & 1.56$\pm$2.95 & 1.80$\pm$2.89  & 1.37$\pm$3.15        & \textbf{1.21$\pm$2.10} \\ \hline
		3  & 3.26$\pm$7.69 & 3.36$\pm$7.76  & 3.12$\pm$6.23        & \textbf{2.40$\pm$4.92} \\ \hline
		5  & 6.33$\pm$5.36 & 6.61$\pm$5.93  & 5.65$\pm$4.33        & \textbf{3.86$\pm$2.82} \\ \hline
		8  & 8.93$\pm$6.11 & 4.98$\pm$4.00  & 4.95$\pm$3.36        & \textbf{4.67$\pm$3.40} \\ \hline
		10 & 9.94$\pm$4.57 & \textbf{4.60$\pm$2.59} & 5.91$\pm$4.52 & 5.84$\pm$3.43           \\ \hline
	\end{tabular}
		\vspace{-1em}
\end{table}

\subsection{Affinity matrix construction by the $\ell_1$-graph }
For the first experiment, we consider the case that the affinity matrix is constructed by the $\ell_1$-graph \cite{elhamifar2013sparse}.  We test on the  Extended Yale B database \cite{YaleBdatabase}  to analyze the effectiveness of our nonconvex SSC model in (\ref{prosmooth}). 	The Extended Yale B dataset consists of 2,414 face images of 38 subjects. Each subject has 64 faces. We resize the images to $32\times 32$   and   vectorized them as 1,024-dimensional data points. We construct 5 subsets which consist of all the images of the randomly selected 2, 3, 5, 8 and 10 subjects of this dataset. For each trial, we follow the settings in \cite{elhamifar2013sparse} to construct the affinity matrix $W$  by solving a sparse representation model  (or $\ell_1$-graph), which is the most widely used method. Based on the learned affinity matrix by $\ell_1$-graph, the following three methods are compared to find the final clustering results:
\begin{itemize}
	\item SC: traditional spectral clustering method \cite{ng2002spectral}.
	\item SSC: convex SSC model \cite{lu2016convex}.
	\item SSC-ADMM: our nonconvex SSC model solved by the proposed ADMM in Algorithm \ref{admm}.
	\item SSC-PG: our nonconvex SSC model (\ref{ncvxssc}) can also be solved by the Proximal Gradient (PG) \cite{beck2009fast} method (a special case of Algorithm 1 in~\cite{mairal2013optimization}). In each iteration, PG updates $\Ukk$ by  
	\begin{align*}
	\Ukk&=\argmin_{U\in\mathbb{R}^{n\times k}}  g(\UUk)+ \inproduct{\nabla g(\UUk)}{UU^\top-\UUk} \\
	&+\frac{l}{2}\norm{UU^\top-\UUk}^2+\inproduct{L}{UU^\top} \\
	&  \st \ \ \ U^\top U = I.
	\end{align*}
	It is easy to see that the above problem has a closed form solution. We use the stopping criteria $\norm{\Uk-\Ukk}_\infty\leq 10^{-6}$. We name the above method as SSC-PG. 
\end{itemize}
Note that all the above four methods use the same affinity matrix as the input and  their main difference is the different ways of learning of low-dimensional embedding $U$. In the nonconvex model (\ref{prosmooth}), we set $\beta=0.01$. The experiments are repeated 20 times and the mean and standard deviation of the clustering error rates  (see the definition in \cite{elhamifar2013sparse}) are reported.

\begin{figure}[!t]
	\centering
	\begin{subfigure}[b]{0.23\textwidth}
		\centering
		\includegraphics[width=\textwidth]{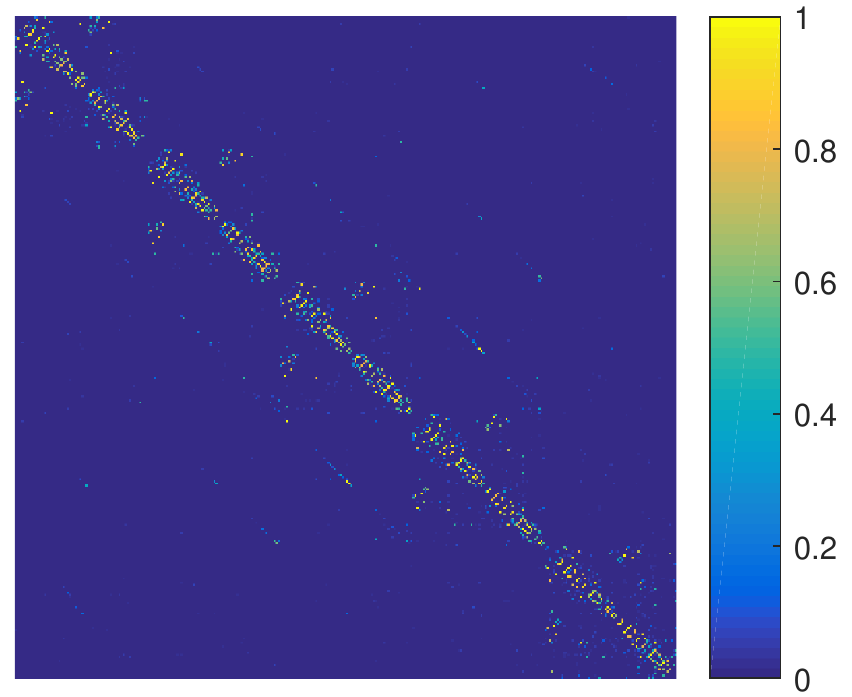}
		\caption{$W$ learned by the $\ell_1$-graph}
	\end{subfigure}
	\begin{subfigure}[b]{0.23\textwidth}
		\centering
		\includegraphics[width=\textwidth]{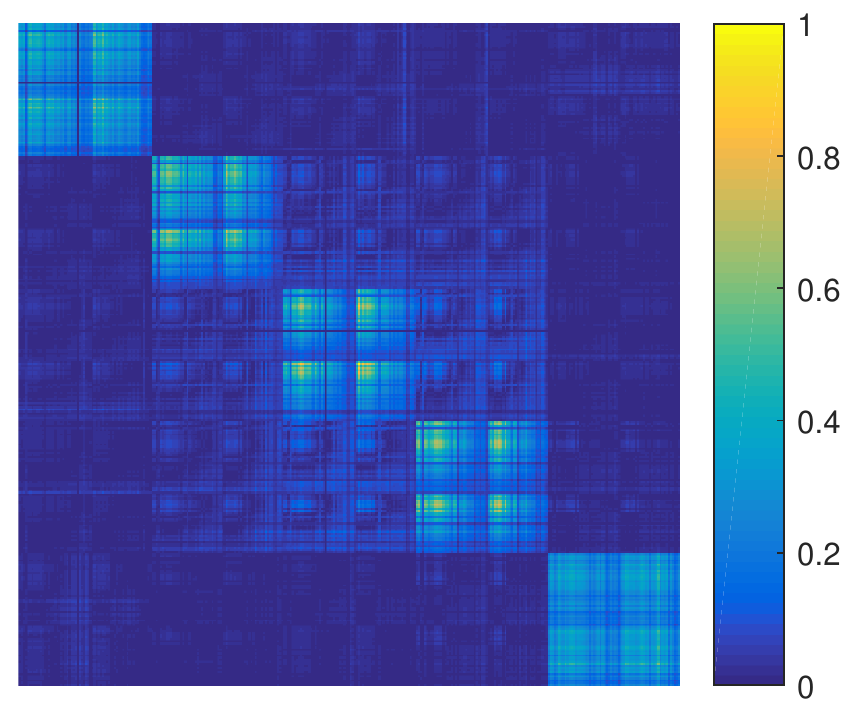}
		\caption{$\UU$ by SC}
	\end{subfigure}
	\begin{subfigure}[b]{0.23\textwidth}
		\centering
		\includegraphics[width=\textwidth]{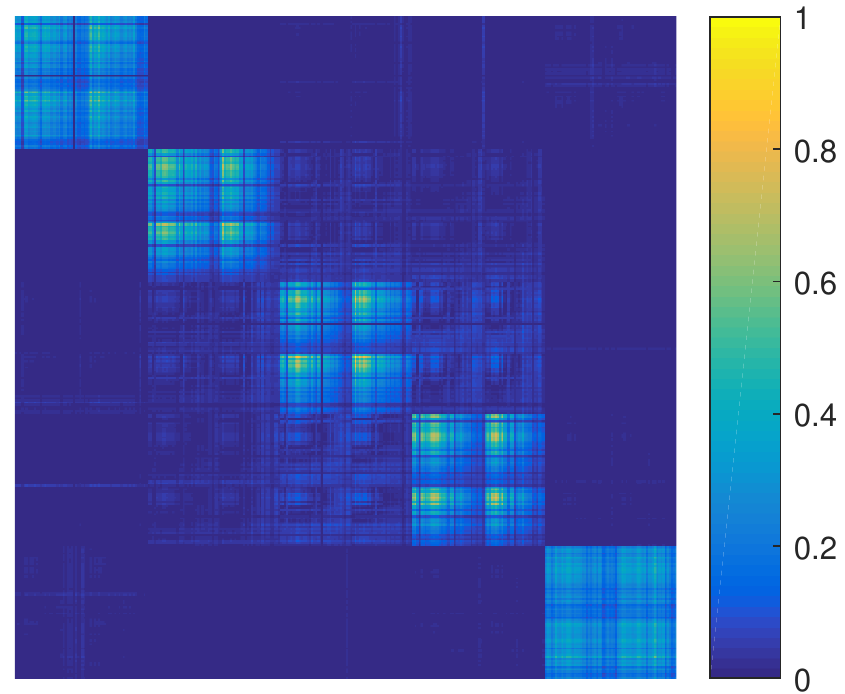}
		\caption{$\UU$ by convex SSC}
	\end{subfigure}
	\begin{subfigure}[b]{0.23\textwidth}
		\centering
		\includegraphics[width=\textwidth]{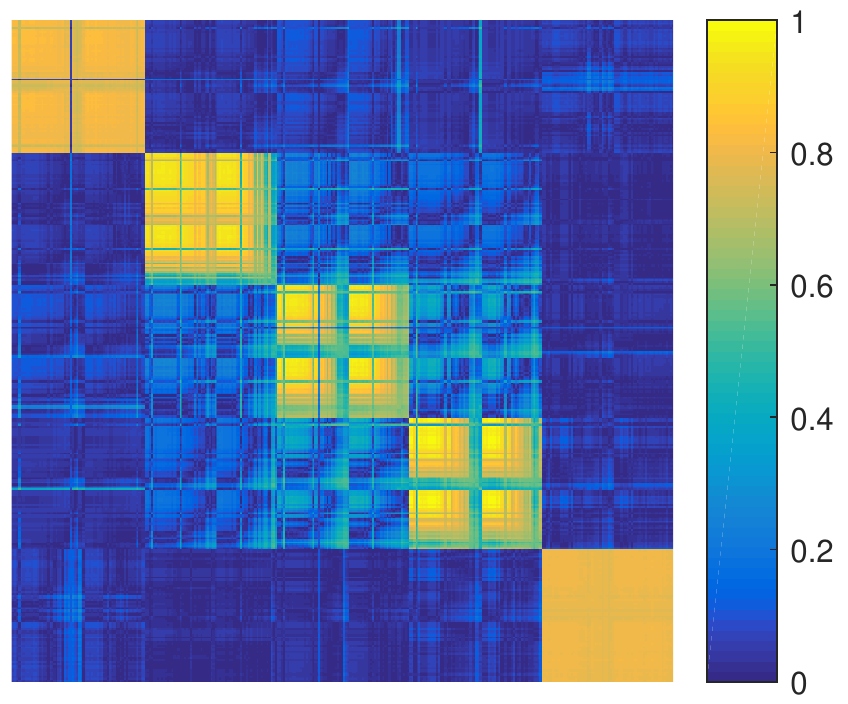}
		\caption{$\UU$ by SSC-ADMM}
	\end{subfigure}
	\caption{ An example with  5 subjects from the Extended Yale B database. (a) Plot of the affinity matrix $W$ learned by the $\ell_1$-graph \cite{elhamifar2013sparse}; (b) Plot of $\UU$ with $U$ learned by SC in (\ref{sc}); (c) Plot of $\UU$ with $U$ learned by convex SSC in (\ref{computUssc}); (d) Plot of $\UU$ with $U$ learned by SSC-ADMM in (\ref{pro}). Each matrix is normalized to [0,1] for better visualization. }\label{fig_matrixyaleb}
		\vspace{-1em}
\end{figure}

\begin{figure}[!t]
	\centering
	\begin{subfigure}[b]{0.3\textwidth}
		\centering
		\includegraphics[width=\textwidth]{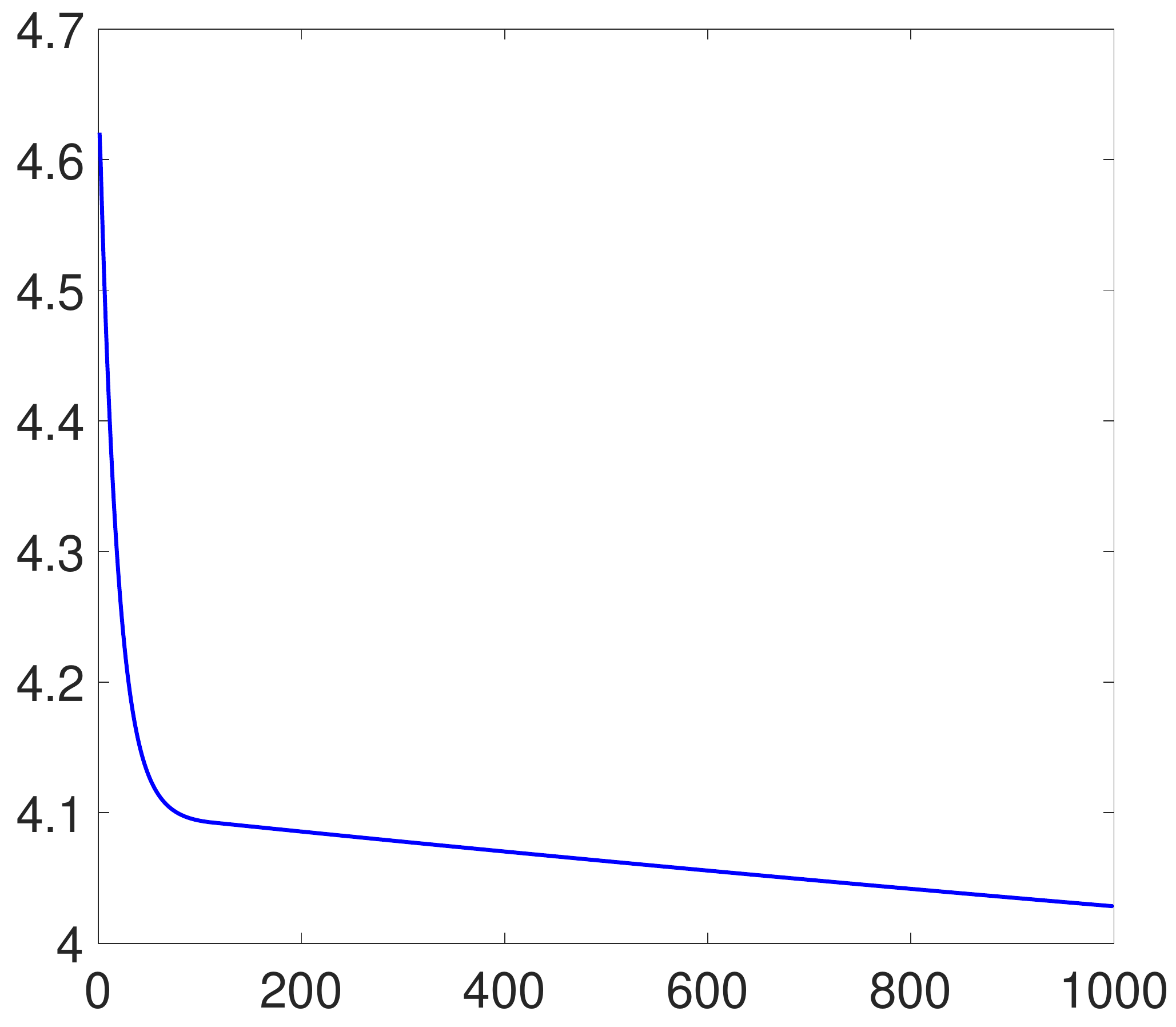}
		\caption{$\Lag(\Pk,\Uk,\Yk,\muk)$ } 
	\end{subfigure}
	\begin{subfigure}[b]{0.3\textwidth}
		\centering
		\includegraphics[width=\textwidth]{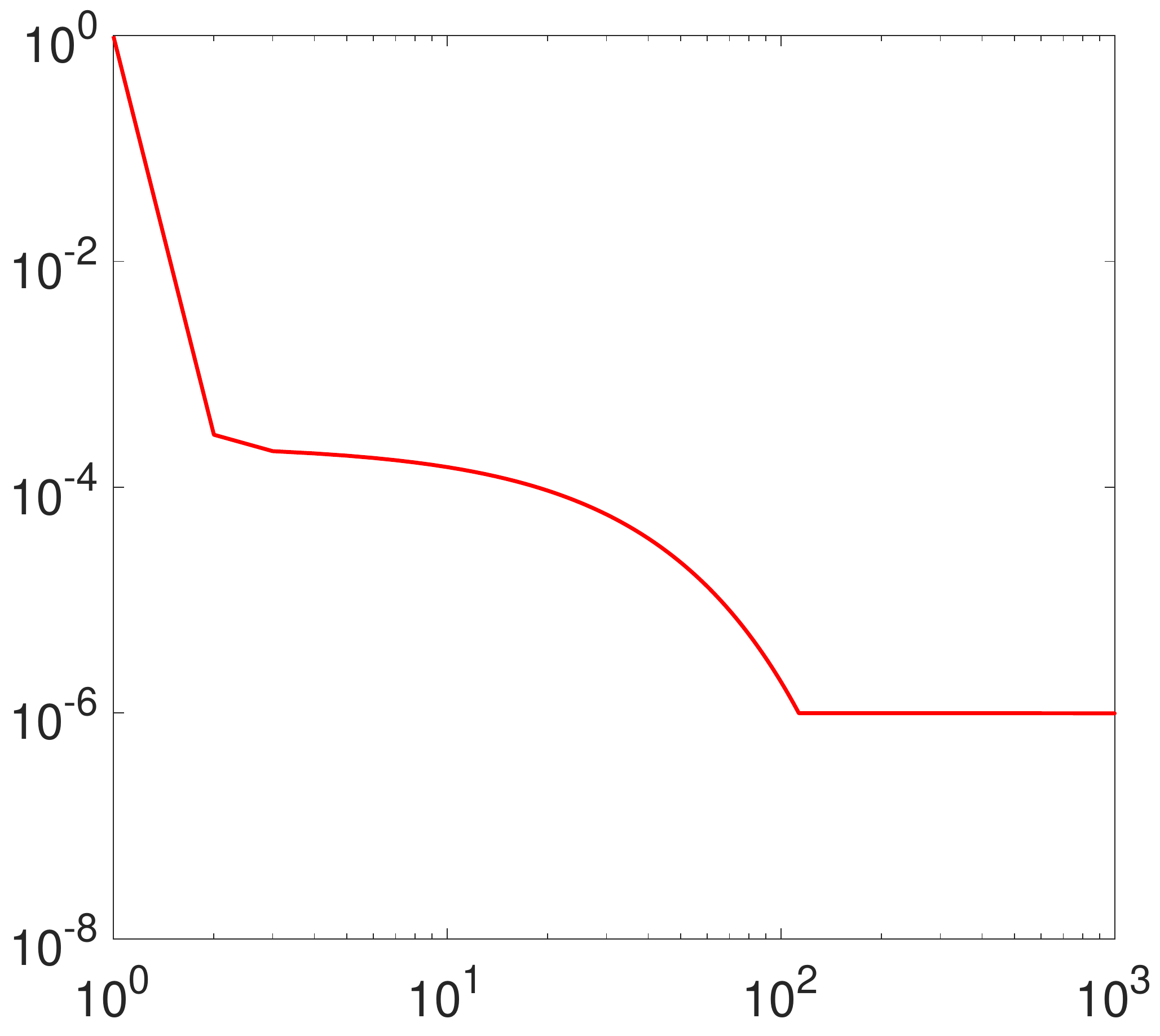}
		\caption{stopping criteria}
	\end{subfigure}
	\caption{   Plots of (a) convergence of   $\Lag(\Pk,\Uk,\Yk,\muk)$ v.s. $k$ and (b) convergence of the stopping criteria in (\ref{stopc}) v.s. $k$. The used data is a 5 subjects subset of   Extended Yale B database. }\label{fig_convergence}
	\vspace{-1em}
\end{figure}

The clustering results  are shown in Table \ref{tab_resyaleb}. 
It can be seen that  our nonconvex SSC models  outperform convex SSC in most cases. The main reason is that  nonconvex SSC is able to directly encourage $\UU$ to be sparse while SSC achieves this in a two-stage way (required solving (\ref{cssc}) and (\ref{computUssc})). Considering a clustering example with $k=5$ subjects  from the Yale B dataset,  Figure \ref{fig_matrixyaleb} plots the learned affinity matrix $W$ by $\ell_1$-graph, and $\UU$ learned by SC, SSC and SSC-ADMM, respectively. Note that $\UU$ is important for data clustering since $\hatU\hatU^\top$ ($\hatU$ is the row normalization of $U$) implies  the true membership of the data clusters in the ideal case (see the discussions in the Introduction section). 
It can be seen that  $\UU$ by SSC-ADMM looks more discriminative since the within-cluster connections are much   stronger than the between-cluster connections.  Also, for the convergence of the proposed ADMM, we plot the augmented Lagrangian function   $\Lag(\Pk,\Uk,\Yk,\muk)$ and the stopping criteria in (\ref{stopc}).  It can be seen that $\Lag$ is monotonically decreasing and the stopping criteria converges towards 0. The convergence behavior is consistent with our theoretical analysis.

\subsection{Affinity matrix construction by the Gaussian kernel }

For the second experiment,  we consider the case that the affinity matrix is constructed by the Gaussian kernel. We test on 10 datasets which are of different sizes and are widely used in pattern analysis. They include 5 datasets (Wine, USPS, Glass, Letter, Vehicle) from the UCI website \cite{ucidataset}\footnote{\url{https://www.csie.ntu.edu.tw/~cjlin/libsvmtools/datasets/}.}, 
UMIST\footnote{\url{https://www.sheffield.ac.uk/eee/research/iel/research/face}}, PIE \cite{sim2002cmu}, COIL20\footnote{\url{http://www1.cs.columbia.edu/CAVE/software/softlib/coil-20.php}}, CSTR\footnote{\url{http://www.cs.rochester.edu/trs/}} and AR~\cite{martinez1998ar}. For some datasets, e.g., USPS and Letter, we use a subset instead due to the relatively large size. For some image datasets, e.g., UMIST, PIE, and COIL20, the images are resized to $32\times 32$ and then vectorized as features.  The statistics of these datasets are summarized in Table \ref{tab_dataset}. The key difference of this experiment from the first one is that we use the Gaussian kernel instead of the sparse subspace clustering to construct the affinity matrix $W$.

\begin{table}[t]	
	\small
	\centering
	\caption{   Statistics of 10 datasets. }
		\vspace{-1.em}
	\begin{tabular}{c|c|c|c}\hline
		dataset     &  $\#$ samples & $\#$ features & $\#$ clusters \\\hline\hline
		Wine & 178 & 13 & 3 \\\hline
		USPS & 1,000 & 256 & 10 \\\hline
		Glass & 214 & 9 & 6 \\ \hline
		Letter & 1,300 & 16 & 26 \\ \hline
		Vehicle & 846 & 18 & 4 \\\hline
		UMIST & 564 & 1,024 & 20 \\\hline
		PIE & 1,428      & 1,024     & 68        \\\hline
		COIL20   & 1,440     & 1,024     & 20        \\\hline	
		CSTR         & 476    & 1,000     & 4        \\\hline
		AR       & 840    & 768     & 120        \\\hline
	\end{tabular}	\label{tab_dataset}
\end{table}

\begin{table}[]
	\small
	\centering
	\caption{  Clustering accuracy on 10 datasets. In this table, SC, SSC, SSC-PG and SSC-ADMM use the \textit{dense} affinity matrix $W$ constructed by the Gaussian kernel.}
	\vspace{-1.em}
	\label{tab_otherdatsets}
	\begin{tabular}{c|c|c|c|c|c|c}
		\cline{1-7}
		\multirow{2}{*}{}            & \multirow{2}{*}{k-means}  & \multirow{2}{*}{NMF}      & \multirow{2}{*}{SC}       & \multirow{2}{*}{SSC}  & SSC- & SSC-                     \\
		&                           &                           &                           &                      & PG & ADMM                       \\ \hline
		Wine                         & 94.4                      & 96.1                      & 94.9                     & \multicolumn{1}{c|}{96.1} & \multicolumn{1}{c|}{96.6} & \multicolumn{1}{c}{\textbf{97.2}} \\ \hline
		USPS                         & 67.3                      & 69.2                      & 71.2                      & \multicolumn{1}{c|}{73.4} & \multicolumn{1}{c|}{\textbf{76.8}} & \multicolumn{1}{c}{\textbf{76.8}} \\ \hline
		Glass                        & 40.7                      & 39.3                      & 41.1                      & \multicolumn{1}{c|}{43.0} & \multicolumn{1}{c|}{44.9} &\multicolumn{1}{c}{\textbf{45.3}} \\ \hline
		Letter                       & 27.1                      & 30.4                      & 31.8                      & \multicolumn{1}{c|}{35.3} & \multicolumn{1}{c|}{\textbf{36.4}}& \multicolumn{1}{c}{{36.0}} \\ \hline
		Vehicle                      & 62.1                      & 61.3                      & 67.0                      & \multicolumn{1}{c|}{70.0} & \multicolumn{1}{c|}{{73.0}} & \multicolumn{1}{c}{\textbf{73.4}} \\ \hline
		\multicolumn{1}{c|}{UMIST}  & \multicolumn{1}{c|}{52.1} & \multicolumn{1}{c|}{63.8} & \multicolumn{1}{c|}{63.3} & \multicolumn{1}{c|}{64.2} & \multicolumn{1}{c|}{{65.1}} & \multicolumn{1}{c}{\textbf{66.1}} \\ \hline
		\multicolumn{1}{c|}{PIE}    & \multicolumn{1}{c|}{35.4} & \multicolumn{1}{c|}{37.9} & \multicolumn{1}{c|}{42.0} & \multicolumn{1}{c|}{46.7} & \multicolumn{1}{c|}{{46.8}} & \multicolumn{1}{c}{\textbf{51.1}} \\ \hline
		\multicolumn{1}{c|}{COIL20} & \multicolumn{1}{c|}{59.0} & \multicolumn{1}{c|}{46.2} & \multicolumn{1}{c|}{63.1} & \multicolumn{1}{c|}{64.5} & \multicolumn{1}{c|}{\textbf{69.1}} & \multicolumn{1}{c}{{67.9}} \\ \hline		\multicolumn{1}{c|}{CSTR}   & \multicolumn{1}{c|}{65.0} & \multicolumn{1}{c|}{70.0} & \multicolumn{1}{c|}{68.9} & \multicolumn{1}{c|}{72.7} & \multicolumn{1}{c|}{{71.0}} & \multicolumn{1}{c}{\textbf{76.3}}\\ \hline
		\multicolumn{1}{c|}{AR}     & \multicolumn{1}{c|}{24.2} & \multicolumn{1}{c|}{35.0} & \multicolumn{1}{c|}{36.1} & \multicolumn{1}{c|}{37.1} & \multicolumn{1}{c|}{{37.7}}  & \multicolumn{1}{c}{\textbf{39.0}}\\ \hline
	\end{tabular}
	\vspace{-1em}
\end{table}

The Gaussian kernel parameter is tuned by the grid $\{10^{-3},10^{-2},10^{-1},10^0\}$.  We use the same affinity matrix $W$ as the input of SC, SSC \cite{lu2016convex}, SSC-PG and SSC-ADMM. The parameter $\beta$ in SSC, SSC-PG and SSC-ADMM is searched from $\{10^{-4},10^{-3},10^{-2}\}$. We further compare the these methods with k-means and Nonnegative Matrix Factorization (NMF).   Under each parameter setting of each method mentioned above, we repeat the clustering for 20 times, and compute the average result. We report the best average accuracy  for each method in Table \ref{tab_otherdatsets}. 

From Table \ref{tab_otherdatsets}, we have the following observations. First, it can be seen that the SSC models (SSC, SSC-PG and SSC-ADMM) improve the traditional SC and our SSC-ADMM  achieves the best performances in most cases. Second, this experiment not only verifies the superiority  of SSC over SC, but also shows the importance of the nonconvex SSC and the effectiveness of our solver. Third, this experiment implies that the nonconvex SSC improves the convex SSC  when given the dense  affinity matrix constructed by    the Gaussian kernel which is different from the sparse $\ell_1$-graph in the first experiment.  Beyond the accuracy, we further compare the efficiency of SSC-PG and SSC-ADMM which use different solvers for the equivalent nonconvex SSC model. Figure \ref{fig_timecomp} gives the average running time of both methods. It can be seen that SSC-ADMM is much more efficient than SSC-PG. The main reason behind is that SSC-PG has to construct a relatively loose majorant surrogate for $g$ in each iteration \cite{mairal2013optimization} and thus requires many more (usually more than $1,000$) iterations. Note that the same phenomenon appears in the convex optimization \cite{lin2010augmented}.

\begin{figure}[!t]
	\centering
	\includegraphics[width=0.32\textwidth]{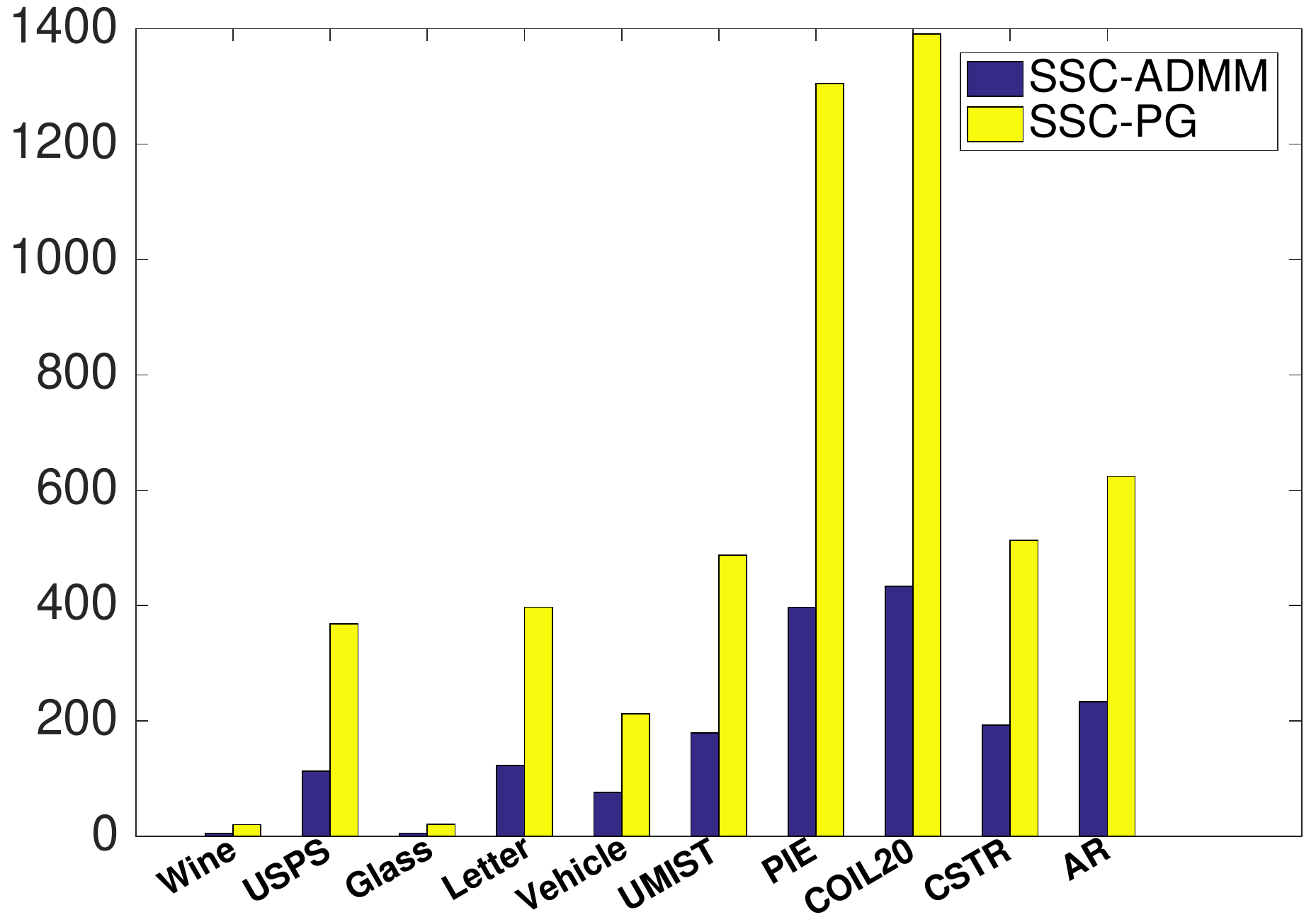}	
	\caption{\small Comparison of the average running time (in second) of SSC-ADMM and SSC-PG on 10 datasets.}\label{fig_timecomp}
\end{figure}

\section{Conclusion and Future Work}
This paper  addressed the loose convex relaxation issue of   SSC  proposed in \cite{lu2016convex}. We proposed to use   ADMM to solve the nonconvex SSC problem (\ref{pro}) directly instead of the convex relaxation. More importantly, we provided the convergence guarantee of ADMM for such a nonconvex problem which is challenging and has not been addressed before. It is worth mentioning that our ADMM method and analysis allow the stepsize to be monotonically increased (but upper bounded). Such a technique has been verified to be effective in improving the convergence in practice for convex optimization and this is the first work which introduces it to ADMM for nonconvex optimization. Also, our convergence guarantee generally requires relatively weak assumptions, e.g., no assumption on the iterates and the subproblems are not necessarily to be strongly convex. Thus it is more practical and can be used to solve other related problems. Beyond the convergence guarantee, we also gave some experimental studies to  verify our analysis and the clustering results on several real datasets demonstrated the improvement of nonconvex SSC over convex SSC.

There have some interesting future works. Though the solved nonconvex problem in this work is specific, the problems with nonconvex  orthogonal constraint  are interesting and such a nonconvex constraint  or related ones appear in many models in component analysis. It will be   interesting to apply ADMM to solve other problems with similar constraints and provide the convergence guarantee.
It will be also interesting to apply our technique to solve some other nonconvex low rank minimization models in \cite{lu2016nonconvex,lu2015generalized}.

\section{Acknowledgements}
J. Feng is partially supported by National University of Singapore
startup grant R-263-000-C08-133 and Ministry of Education of
Singapore AcRF Tier One grant R-263-000-C21-112. Z. Lin is
supported by National Basic Research Program of China (973 Program) (Grant no. 2015CB352502) and National Natural Science
Foundation (NSF) of China (Grant nos. 61625301 and 61731018).

{
	
	\bibliographystyle{aaai}
	\bibliography{bibfile}
}

\newpage
\onecolumn
~\
\begin{center}
	\LARGE\textbf{Supplementary Material}
\end{center}



\section*{Proof of Lemma \ref{lem3}}
\begin{proof} \textbf{Proof of (a).}
	We deduce
	\begin{align}
	&\Lag(\Pkk,\Ukk,\Ykk,\mukk) - \Lag(\Pk,\Uk,\Yk,\muk) \notag \\ 
	= & \Lag(\Pkk,\Ukk,\Ykk,\mukk) - \Lag(\Pkk,\Ukk,\Yk,\muk) \notag\\
	& + \Lag(\Pkk,\Ukk,\Yk,\muk)  - \Lag(\Pk,\Uk,\Yk,\muk).\label{lem3prf1}
	\end{align}
	Consider the first two terms in (\ref{lem3prf1}), we have
	\begin{align}
	& \Lag(\Pkk,\Ukk,\Ykk,\mukk) - \Lag(\Pkk,\Ukk,\Yk,\muk) \notag \\
	=  & \inproduct{\Ykk-\Yk}{\Pkk-\UUkk}  \notag \\
	& + \frac{\mukk-\muk}{2} \norm{\Pkk-\UUkk}^2 \notag\\
	\overset{\text{\ding{172}}}{=} & \left( \frac{1}{\muk} + \frac{\mukk-\muk}{2\muk^2} \right) \norm{\Ykk-\Yk}^2\notag \\
	\overset{\text{\ding{173}}}{\leq} & \frac{\rho+1}{2\muk} \norm{\Ykk-\Yk}^2 
	\overset{\text{\ding{174}}}{\leq}  \frac{l^2(\rho+1)}{2\muk} \norm{\Pkk-\Pk}^2,   \label{lem3prf3}
	\end{align}
	where \ding{172} uses (\ref{updateY}), \ding{173} uses the fact $\mukk \leq \rho \muk$ due to (\ref{updatemu}), and \ding{174} uses (\ref{boundYdiff}). 
	
	Now, let us bound the last two terms in  (\ref{lem3prf1}). By the optimality of $\Ukk$ to problem (\ref{updateU}), we have
	\begin{equation} \label{lem3prf44}
	\Lag(\Pk,\Ukk,\Yk,\muk)   \leq  \Lag(\Pk,\Uk,\Yk,\muk). 
	\end{equation}
	Consider the optimality of $\Pkk$ to problem (\ref{updateP}), note that $\Lag(P,\Ukk,\Yk,\muk)$ is strongly convex with modulus $\gammak$, we have
	\begin{align}
	&\Lag(\Pkk,\Ukk,\Yk,\muk)\notag \\
	\leq &  \Lag(\Pk,\Ukk,\Yk,\muk) - \frac{\gammak}{2}\norm{\Pkk-\Pk}^2, \label{lem3prf5}
	\end{align}
	where we uses the Lemma B.5  in \cite{mairal2013optimization}.

	Combining (\ref{lem3prf1})-(\ref{lem3prf5}) leads to
	\begin{align*}
	& \Lag(\Pkk,\Ukk,\Ykk,\mukk) - \Lag(\Pk,\Uk,\Yk,\muk)   \notag \\ 
	\leq & - \left( \frac{\gammak}{2} - \frac{l^2(\rho+1)}{2\muk}  \right) \norm{\Pkk-\Pk}^2.  
	\end{align*}
	By the choice of $\muk$ and $\gammak$ in assumption A3 and (\ref{profdecresL}), we can see that $\Lag(\Pk,\Uk,\Yk,\muk)$ is monotonically decreasing.
	
	\textbf{Proof of (b).}
	To show that $\Lag(\Pk,\Uk,\Yk,\muk)$ converges to some constant $\Lag^*>-\infty$, we only need to show that $\Lag(\Pk,\Uk,\Yk,\muk)$ is lower bounded. Indeed,
	\begin{align*}
	& \Lag(\Pkk,\Ukk,\Ykk,\mukk) \\
	= & \inproduct{L}{\UUkk} + g(\Pkk) + \inproduct{\Ykk}{\Pkk-\UUkk} \notag \\
	&+ \frac{\mukk}{2} \norm{\Pkk-\UUkk}^2\\ 
	\overset{\text{\ding{175}}}{=} & \inproduct{L}{\UUkk} + \inproduct{\nabla g(\Pkk)}{\UUkk-\Pkk} \notag\\
	& + g(\Pkk) + \frac{\mukk}{2} \norm{\Pkk-\UUkk}^2\\
	\overset{\text{\ding{176}}}{\geq} & \inproduct{L}{\UUkk} + g(\UUkk).
	\end{align*}
	where \ding{175} uses (\ref{optcondP})  and \ding{176} uses the Lipschitz continuous gradient property of $g$ and $\mukk\geq l$ by assumption A3.
	Note that $\inproduct{L}{\UUkk} \geq0$ since $L\succeq 0$ by assumption A1. This combines with the lower bounded assumption of $g$ in assumption A2 implies that $L(\Pkk,\Ukk,\Ykk,\mukk)$ is lower bounded.
	
	\textbf{Proof of (c).} Summing over both sides of (\ref{profdecresL}) from 0 to $+\infty$ leads to
	\begin{align*}
	& \sum_{k=0}^{+\infty} \left( \frac{\gammak}{2} - \frac{l^2(\rho+1)}{2\muk}  \right) \norm{\Pkk-\Pk}^2 \\
	\leq &\Lag(P_0,U_0,Y_0,\mu_0) - \Lag^*.
	\end{align*}
	This implies that $\Pkk-\Pk\rightarrow0$ under assumption A3. Thus $\Ykk-\Yk\rightarrow0$ due to (\ref{boundYdiff}). Finally, $\Pkk-\UUkk = \frac{1}{\muk}(\Ykk-\Yk)\rightarrow0$ since $\muk$ is bounded ($\mu_0\leq\muk\leq \mu_{\max}$).
	
	\textbf{Proof of (d).} First, it is obvious that $\{\Uk\}$ is bounded due to the constraint $U_k^\top U_k = I$.  Thus, $\UUk$ is bounded. Then, we deduce
	\begin{align*}
	\norm{\Pk} = \norm{\Pk-\UUk+\UUk} \leq \norm{\Pk-\UUk} + \norm{\UUk}. 
	\end{align*}
	Note that $\norm{\Pk-\UUk}$ is   bounded since $\Pk-\UUk\rightarrow0$. Hence, $\{\Pk\}$ is bounded. Considering that $\nabla g(P)$ is Lipschitz continuous, i.e., $\norm{\nabla g(P_{k_1})-\nabla g(P_{k_2})}\leq l \norm{P_{k_1}-P_{k_2}}$, this implies that $\nabla g(\Pk)$ is bounded. Thus, $\{\Yk\}$ is bounded due to (\ref{optcondP}).
	
	\textbf{Proof of (e).} First, from the optimality of $\Ukk$ to problem (\ref{updateU}), there exists $G_O\in\partial_U\iota_O(\Ukk)$ such that 
	\begin{align*}
	&{\partial_U  \Lag(\Pk,\Ukk,\Yk,\muk)} + G_O  \\
	= &2 (L-\Yk-\muk\Pk)\Ukk + G_O =0.
	\end{align*} 
	Thus, accordingly, there exists $G_U$ such that 
	\begin{align}
	&G_U = {\partial_U \Lag(\Pkk,\Ukk,\Ykk,\muk)}  + G_O  \notag\\
	= & 2 (L-\Ykk-\muk\Pkk)\Ukk + G_O \notag \\
	= & 2 (L-\Yk-\muk\Pk)\Ukk + G_O + 2(\Yk-\Ykk)\Ukk \notag \\
	&+2\muk(\Pk-\Pkk)\Ukk \notag\\
	= & 2((\Yk-\Ykk)+\muk(\Pk-\Pkk))\Ukk. \label{profeu}
	\end{align}
	Second, by using the optimality  of $\Pkk$  given in (\ref{optcondPP}), we have
	\begin{align}
	&G_P =  {\partial_P \Lag(\Pkk,\Ukk,\Ykk,\muk)}  \notag\\ 
	= & \nabla g(\Pkk) + \Ykk + \muk(\Pkk-\UUkk) \notag\\
	= & \nabla g(\Pkk) + \Yk + \muk(\Pkk-\UUkk) + \Ykk-\Yk \notag\\
	= & \Ykk-\Yk. \label{profep}
	\end{align} 
	Third, by direction computation, we have
	\begin{align}
	& G_Y =  {\partial_Y \Lag(\Pk,\Ukk,\Yk,\muk)} \notag \\
	= & \Pkk-\UUkk
	= \frac{1}{\muk}(\Ykk-\Yk).\label{profey}
	\end{align}
	Finally, combing (\ref{profeu})-(\ref{profey}), we obtain
	\begin{align*}
	&\norm{G}^2 =\norm{[G_P \  G_U \ G_Y]}^2 \\
	\leq & \norm{2((\Yk-\Ykk)+\muk(\Pk-\Pkk))\Ukk}^2 \notag\\
	&+ (1+\frac{1}{\muk^2}) \norm{\Ykk-\Yk}^2 \\
	\leq & (8d+1+\frac{1}{\muk^2})\norm{\Yk-\Ykk}^2 + 8d\muk^2\norm{\Pk-\Pkk}^2\\
	\leq & (8d+1+\frac{1}{\mu_0^2})\norm{\Yk-\Ykk}^2 + 8d\mu_{\max}^2\norm{\Pk-\Pkk}^2.
	\end{align*}
	The proof is completed.
\end{proof}

\end{document}